\documentclass{article}

\usepackage{authblk}
\usepackage{natbib}
\usepackage{hyperref}
\AtBeginDocument{}

\usepackage{url}  
\usepackage{graphicx}  
\usepackage{float}
\usepackage{amsmath}
\usepackage{amssymb}
\usepackage{amsthm}
\floatstyle{ruled}
\newfloat{algorithm}{tbp}{loa}
\providecommand{\algorithmname}{Algorithm}
\floatname{algorithm}{\protect\algorithmname}

\theoremstyle{plain}
\newtheorem{thm}{\protect\theoremname}
\theoremstyle{definition}

\theoremstyle{plain}
\newtheorem{lem}[thm]{\protect\lemmaname}
\theoremstyle{plain}

\ifx\proof\undefined
\newenvironment{proof}[1][\protect\proofname]{\par
	\normalfont\topsep6\p@\@plus6\p@\relax
	\trivlist
	\itemindent\parindent
	\item[\hskip\labelsep\scshape #1]\ignorespaces
}{%
	\endtrivlist\@endpefalse
}
\providecommand{\proofname}{Proof}
\fi

\usepackage{algorithmic}

\providecommand{\corollaryname}{Corollary}
\providecommand{\definitionname}{Definition}
\providecommand{\lemmaname}{Lemma}
\providecommand{\theoremname}{Theorem}
\providecommand{\propositionname}{Proposition}

\begin{document}

\title{Robust Gaussian Process Regression for Real-Time High Precision GPS Signal Enhancement}

\author[1]{Ming Lin \thanks{ming.l@alibaba-inc.com}}
\author[1]{Xiaomin Song \thanks{xiaomin.song@alibaba-inc.com}}
\author[1]{Qi Qian \thanks{qi.qian@alibaba-inc.com}}
\author[2]{Hao Li \thanks{lihao.lh@alibaba-inc.com}}
\author[1]{Liang Sun \thanks{liang.sun@alibaba-inc.com}}
\author[1]{Shenghuo Zhu \thanks{shenghuo.zhu@alibaba-inc.com}}
\author[1]{Rong Jin \thanks{jinrong.jr@alibaba-inc.com}}
\affil[1]{Alibaba Group, Bellevue, Washington, USA.}
\affil[2]{Alibaba Group, Hangzhou, Zhejiang, China.}

%

\maketitle
\begin{abstract}
  Satellite-based positioning system such as GPS often suffers from large amount of
  noise that degrades the positioning accuracy dramatically especially in real-time
  applications. In this work, we consider a data-mining approach to enhance the GPS
  signal. We build a large-scale high precision GPS receiver grid system to collect
  real-time GPS signals for training. The Gaussian Process (GP) regression is chosen to
  model the vertical Total Electron Content (vTEC) distribution of the ionosphere of the
  Earth. Our experiments show that the noise in the real-time GPS signals often exceeds
  the breakdown point of the conventional robust regression methods resulting in
  sub-optimal system performance. We propose a three-step approach to address this
  challenge. In the first step we perform a set of signal validity tests to separate the
  signals into clean and dirty groups. In the second step, we train an initial model on
  the clean signals and then reweigting the dirty signals based on the residual error. A
  final model is retrained on both the clean signals and the reweighted dirty signals.
  In the theoretical analysis, we prove that the proposed three-step approach is able to
  tolerate much higher noise level than the vanilla robust regression methods if two
  reweighting rules are followed. We validate the superiority of the proposed method in
  our real-time high precision positioning system against several popular
  state-of-the-art robust regression methods. Our method achieves centimeter positioning
  accuracy in the benchmark region with probability $78.4\%$ , outperforming the second
  best baseline method by a margin of $8.3\%$. The benchmark takes 6 hours on 20,000 CPU
  cores or 14 years on a single CPU.
\end{abstract}


\section{Introduction}

Real-time high precision positioning service is a critical core component
in modern AI-driven industries, such as self-driving cars, autonomous
drones and so on. The Global Positioning System (GPS) is inarguably
the most popular, if not the only for now, satellite-based positioning
system accessible to public all around the world. The major performance
index of GPS based system is the positioning accuracy. When using
civilian level smart phones with GPS enabled, the horizontal positioning
accuracy is around 5 meters. Even equipped with high quality, single
frequency GPS receiver, the horizontal accuracy is typically within
1.891 meters with 95\% probability\footnote{Data source \url{https://www.gps.gov/systems/gps/performance/accuracy/}}.
While it might be sufficient for daily usage, the meter level accuracy
is insufficient for many modern autonomous applications. Previously
a popular method to achieve the centimeter accuracy is to fix the
GPS receiver at one point for several hours. This method is clearly
infeasible in mobile applications. How to achieve centimeter accuracy
in real-time is therefore an open challenge in this domain.

In this work, we develop a data mining approach to address the above
challenge. The majority of the noise in GPS signal is the signal delay
caused by the ionosphere which can be described by the vertical Total
Electron Content (vTEC). To eliminate this delay, we build a grid
system where each node in the grid is a ground station equipped with
high precision multiple frequency GPS receiver. A Gaussian Process
model (GP) is learned to predict the vTEC value for any given geographic
coordinate. When a mobile client requests positioning service, its
GPS signal is calibrated by our GP model every second.

A critical problem in the model training step is the large amount
of noisy data. As validated in our experiment, there are 20\%$\sim$40\%
outliers in the GPS signal received by our grid system. Directly applying
off-the-shelf robust regression methods cannot achieve the required
performance because existing methods are either non-consistent or
are not suitable to tolerate high corruption rate. To overcome this
difficulty, we develop a screening algorithm to detect outliers and
split the received signal into clean and noisy subsets. The GP model
is then trained on the clean dataset.

However, only training on the clean dataset has an obvious drawback.
The screening algorithm is not always reliable. It often over-kills
clean data. If the dirty dataset is completely discarded after screening,
we cannot collect sufficient clean data for robust prediction. On
the other hand, each data point is collected from an expensive high
precision GPS receiver therefore simply discarding the noisy dataset
is a great waste. In the worst case nearly 40\% data points will be
marked as noisy in one day. This number is even higher $(>60\%)$
at noon in summer which results in a low positioning accuracy or even
interruption of our service.

In order to address the above problem, we formulate the problem as
a robust regression problem. We are considering a scenario where we
are provided a clean dataset and a noisy dataset by an unreliable
screening algorithm. The noise in the clean dataset is sub-gaussian
while the noise in the noisy dataset is heavy-tailed. The volume of
the noisy dataset might be infinite. We aim to design an efficient
and robust algorithm to boost the model performance by learning on
both clean and noisy datasets. Under this setting, we find that existing
robust regression methods cannot be applied directly due to their
small breakdown points or due to their inconsistency. To this end,
we propose a three-step algorithm. Our algorithm first learns an initial
estimator only using the clean dataset. Then we apply the estimator
on the noisy dataset to filter out those of large residual error.
The remaining noisy data points are reweighted according to their
residual error and finally a robust estimator is retrained on the
reweighted noisy data points in addition to the clean dataset together.
We call this approach as Filter-Reweight-Retrain (FRR). Our theoretical
analysis shows that the three steps in FRR is\emph{ not only sufficient
but actually necessary}. The filtering step truncates the noise level
in the noisy dataset. The reweighting step reduces the variance of
the noise. When the volume of the noisy dataset is infinite, FRR is
consistent which means that it achieves zero recovery error. When
the volume of the noisy dataset is finite, a reweighting scheme is
proposed to improve the asymptotic recovery error bound. While many
previous robust regression methods involve iterative retraining and
reweighting, we show that FRR does not need multiple iterations. In
fact, our analysis suggests that designing an effective reweighting-retraining
scheme is non-trivial due to the risk of over-fitting. We give two
general rules to avoid over-fitting in the reweighting step.

The remainder of this work is organized as following. In Section 2
we briefly survey related works. Section 3 introduces some background
knowledge of our GPS positioning system and our main algorithm, robust
Gaussian process regression with FRR. Theoretical analysis is given
in Section 4. We demonstrate the performance of our method in real-time
GPS positioning system in Section 5. Section 6 encloses this work.

\section{Related Work}

The vTEC contributes the majority of noise in the GPS signal received
by ground stations. Conventional ionosphere scientific researches
focus on the static estimation of vTEC. For example, \cite{sardon1994estimation,mannucci1998global}
used Kalman filter to vTEC based on GPS signals received by multiple
ground stations.\cite{arikan2003regularized} proposed a high pass
penalty regularizer to smooth the estimated vTEC values. \cite{Mukesh2019}
described the vTEC distribution by an ordinary kriging (OK)-based
surrogate model. The real-time vTEC estimation was considered in \cite{huang2001vertical}.
In \cite{Renga2018}, the authors proposed a Linear Thin Shell model
to better describe the horizontal variation of the vTEC distribution
in real-time. Several researches introduce multiple receivers to jointly
estimate the vTEC. In \cite{200263}, the authors used over 1000 dual
frequency receivers to construct a large-scale vTEC map over Japan.
\cite{Zhang2018} used low-cost receivers jointly to improve the vTEC
estimation quality. Comparing to previous researches, we use robust
GP regression to model the vTEC distribution. The GPS signal is collected
from a few hundreds of ground stations in a given region. With modern
hardware and our new algorithm, we are able to report the centimeter
positioning accuracy in real-time.

In regard to the robust regression, we briefly survey recent works
closely related to ours. As there is a vast amount of robust regression
literature, we refer to \cite{Morgenthaler2007} for a more comprehensive
review.

When comparing robust regression methods, we usually consider the
performance of an algorithm from three aspects: breakdown point, consistency
and corruption model. The breakdown point is a real number which measures
the maximum ratio of corruption data the algorithm can tolerance.
A more robust method should have a larger breakdown point. The consistency
tells whether the algorithm is unbiased. An unbiased algorithm should
achieve zero recovery error when provided with infinite training data.
The corruption model makes assumptions about the noise. The conventional
setting assumes that the training data is randomly corrupted by arbitrary
large noise. A more stronger corruption model is the oblivious adversarial
model which assumes that the attacker can add noise at arbitrary location
without looking at the data. The most difficult model is the adaptive
adversarial model. In this model, the attacker is able to review the
data and then adaptively corrupts the data.

Ideally, we wish to find a robust algorithm with breakdown point asymptotically
convergent to $1$ while being consistent under adaptive adversarial
corruption. However, it is impossible to satisfy all the three requirements
\cite{Diakonikolas2018}. A popular approach is to use $\ell_{1}$-norm
in the loss function. The $\ell_{1}$-norm based methods usually have
breakdown point asymptotically convergent to 1 but inconsistent \cite{Wright2010,Nguyen2013}.
\cite{Dalalyan2012} proposed a \emph{second order cone programming}
method (SOCP). The SOCP has breakdown point asymptotically convergent
to $1$ but is inconsistent. \cite{McWilliams2014} proposed to use
weighted subsampling and then iteratively retrain the model. Its breakdown
point is $\mathcal{O}(1/\sqrt{d})$ and the algorithm is known to
be inconsistent. Recently, a branch of \emph{hard iterative thresholding}
(HIT) based methods are proven to be more efficient. \cite{Chen2013}
shows that HIT has breakdown point on order of $\mathcal{O}(1/\sqrt{d}$).
\cite{Bhatia2015} proposed a variant of HIT with breakdown point
convergent to $1$ but inconsistent. The first consistent HIT algorithm
is proposed by \cite{Bhatia2017} with breakdown point $\mathcal{O}(10^{-4})$.
The corruption models in the above works are either oblivious or adaptive.

Comparing to the above works, our setting is new where we are given
two datasets to train our model rather than one mixed with clean and
corrupted data. In our setting, the volume of the noisy dataset could
be infinite which means the breakdown point of our algorithm is asymptotically
convergent to 1. Our algorithm is consistent when the volume of the
noisy dataset is infinite.

\section{Our System and Algorithm}

This section consists of two subsections. In subsection \ref{subsec:How-GPS-positioning},
we describe our Global Navigation Satellite System (GNSS) behind the
scene. We give our FRR algorithm in subsection \ref{subsec:Filter-Reweight-Retrain-Meta-Alg}.

\subsection{Global Navigation Satellite System \label{subsec:How-GPS-positioning}}

\begin{figure*}
\begin{minipage}[b][][b]{0.32 \linewidth}
\includegraphics[width=\linewidth]{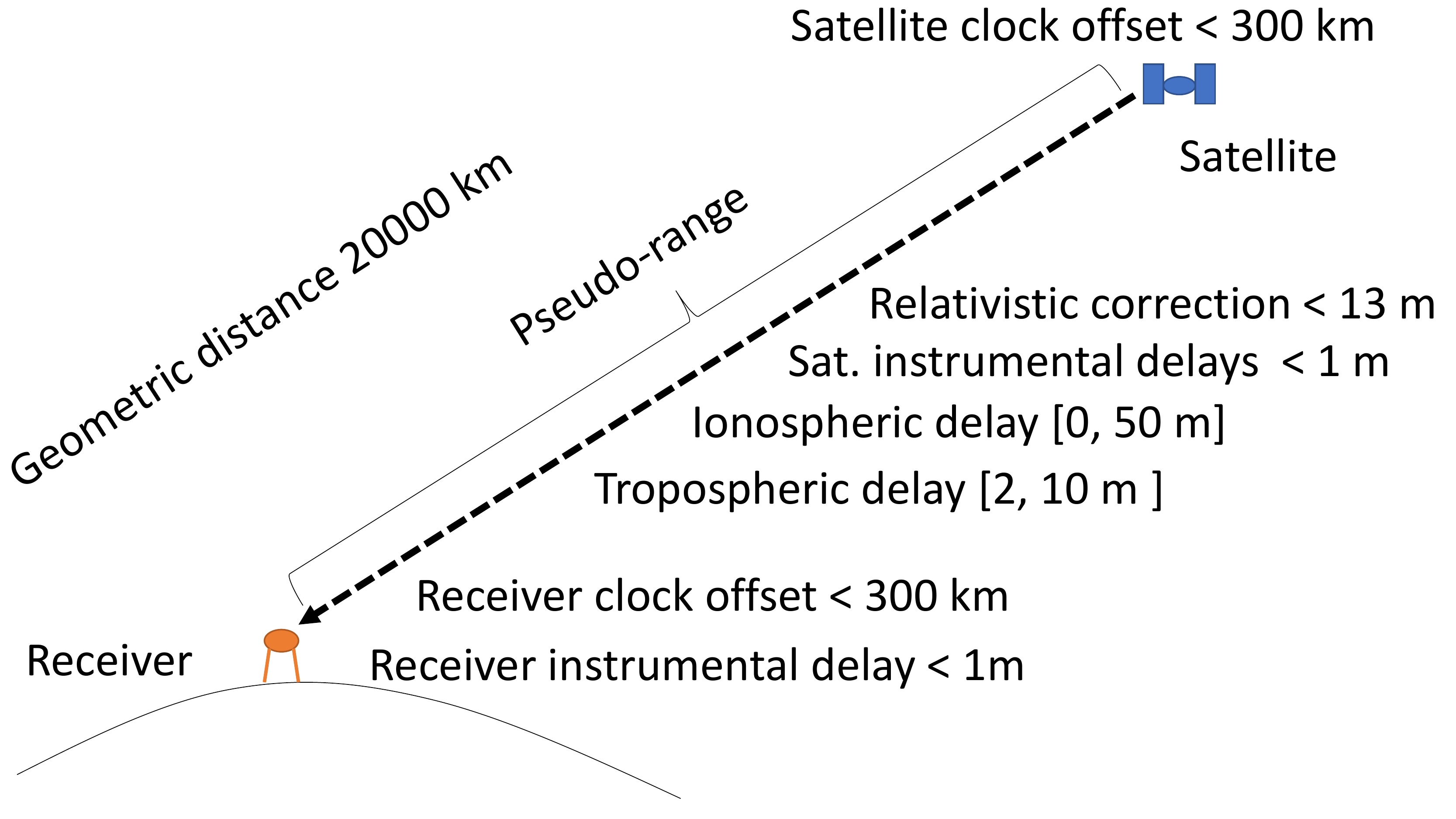}\\
\begin{center}
(a)
\end{center}
\end{minipage} \hfil
\begin{minipage}[b][][b]{0.32 \linewidth}
\includegraphics[width=\linewidth]{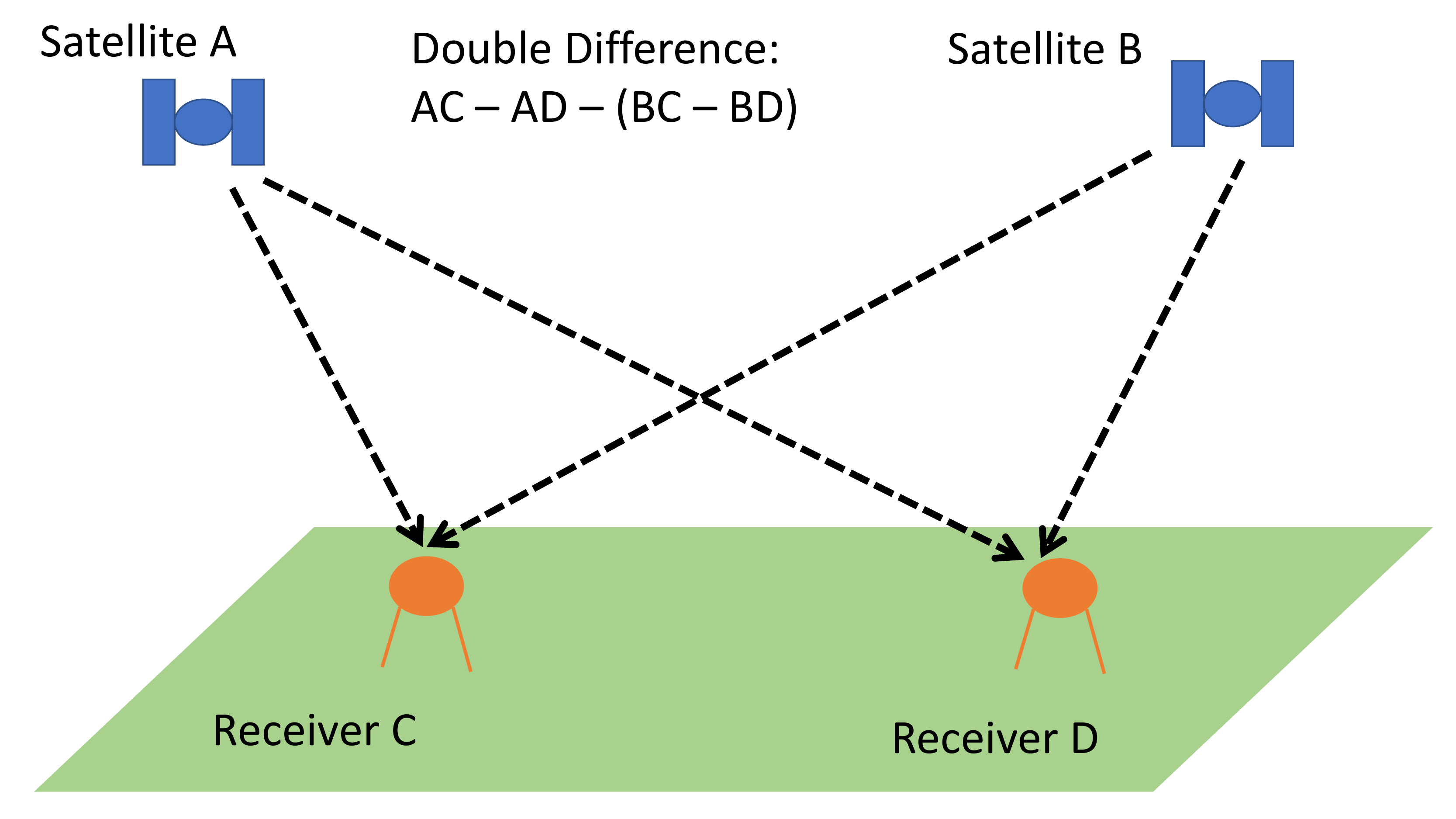}\\
\begin{center}
(b)
\end{center}
\end{minipage} \hfil
\begin{minipage}[b][][b]{0.32 \linewidth}
\includegraphics[width=\linewidth]{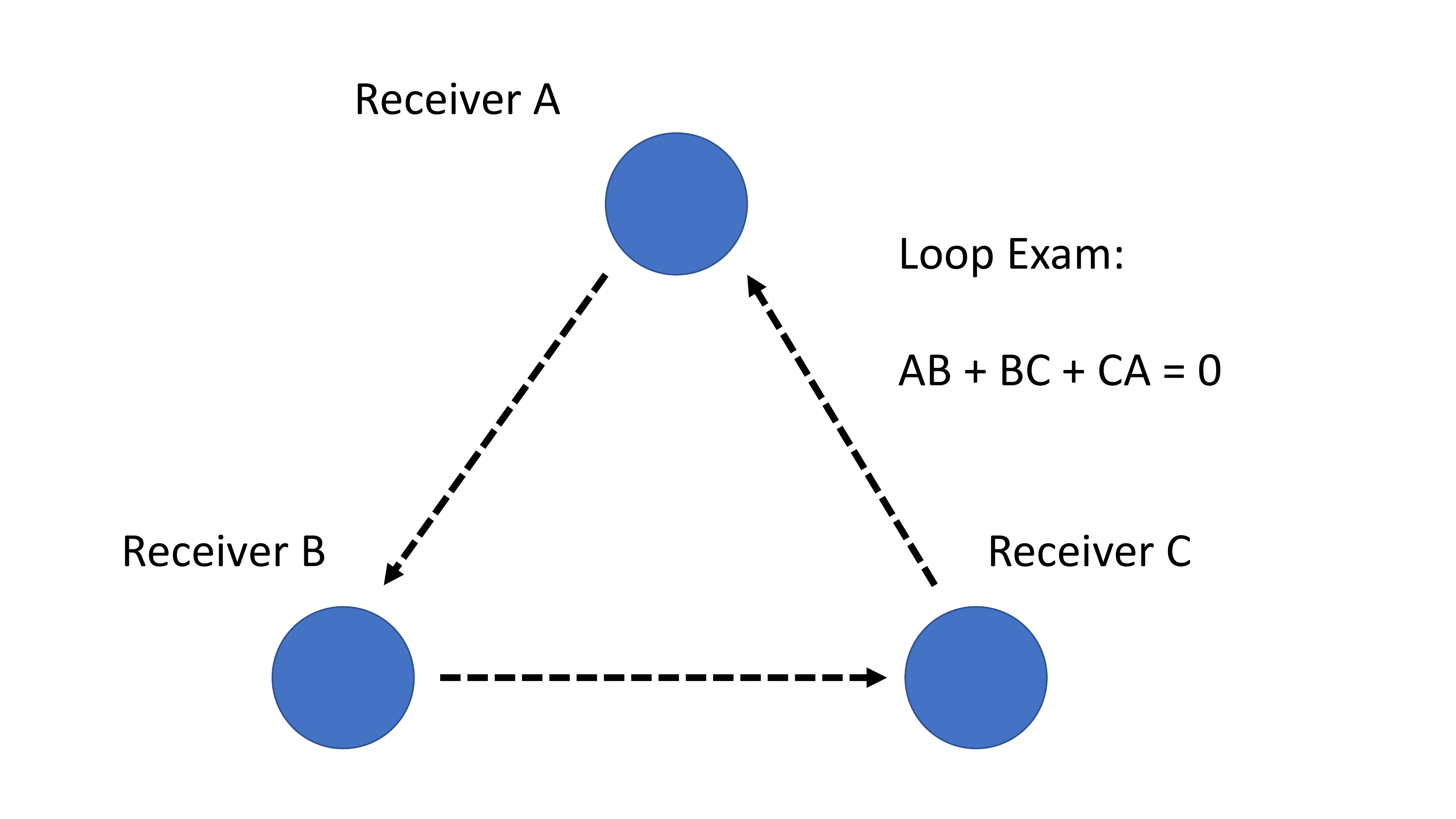}\\
\begin{center}
(c)
\end{center}
\end{minipage}\caption{GNSS Background}
\label{fig:gnss_background}
\end{figure*}

A global navigation satellite system (GNSS) is a satellite navigation
system that provides positioning service with global coverage, including
GPS from the United States, GLONASS from Russia, BeiDou from China
and Galileo from European Union. In GNSS, the satellite positions
are known. With high precision distance measurement between user and
more than 4 satellites, the user\textquoteright s location can be
determined \cite{Hofmann-Wellenhof1993}. The distance is calculated
from the time interval between signal emitting from the satellite
and signal receiving by the user, times the speed of microwave transmission
through the path from the satellite to the user.

In a high precision positioning system, the major error terms along
the path of microwave transmission are shown in Figure \ref{fig:gnss_background}(a).
The pseudorange, a measure of distance between the satellite and the
receiver, is obtained through the correlation of the modulated code
in the received signal from the satellite with the same code generated
in the receiver. The pseudorange is affected by the following factors:
\begin{itemize}
\item The Euclidean distance between satellite position at signal emission
and the receiver at the signal reception.
\item Offsets of receiver clock and satellite clock. They are the clock
synchronism errors of the receiver and satellite referring to GNSS
time. They are associated with the specific receiver or satellite.
\item Relativistic correction which can be modeled very well.
\item Tropospheric delay which is frequency independent.
\item Ionospheric delay which is frequency dependent.
\item Instrumental delay which is frequency dependent and can be model very
well.
\item Multi-path effect. This error can be minimized by high quality antenna.
\end{itemize}
Most of errors like clock error and instrumental delay can be mitigated
by double differencing and multiple frequencies \cite{Pajares2005}.
The computation of double difference is shown in Figure \ref{fig:gnss_background}(b).
Biases due to orbit and atmosphere are significantly reduced due to
the similar path of transmissions. As the ionospheric delays are spatially
correlated, the between-receiver differenced ionospheric bias is much
less than its absolute values.

After double difference, the residual tropospheric delays are rather
small since the tropospheric delays can be at least 90\% corrected
using the empirical model \cite{collins1997tropospheric}. The remaining
error is dominated by the modeling of the ionosphere. The ionosphere
is in the Earth's upper atmosphere. It can be modeled as a thin layer
of ionized electrons at 350 kilometers above sea level \cite{Klobuchar1987}.
The vertical Total Electron Content (vTEC) is an important descriptive
parameter of the ionosphere. It integrates the total number of electrons
along a perpendicular path through the ionosphere. The additional
phase and group delay of microwave signal accumulated during transmission
through ionosphere are proportional to the vTEC. The vTEC is impacted
by the nearly unpredictable solar flare which brings inherently high
temporal and spatial fluctuations \cite{Meza2009}.

To estimate the vTEC, we use machine learning to learn the vTEC distribution
from double differences received in real-time. However due to the
difficulties in signal resolving, the quality of the received data
varies a lot. To understand this, we must understand how the signal
is resolved. First we solve the double differential equations with
both code and phase using wide-lane combination \cite{Teunissen1998}.
The equations are solved with Kalman filter to get floating ambiguity
\cite{teunissen2003towards}. Then the narrow-lane float ambiguity
is derived from ionosphere-free and wide-lane combination. The integer
phase ambiguity can be found by Least-squares AMBiguity Decorrelation
Adjustment (LAMBDA) method \cite{Teunissen1995}. After finding the
integer ambiguity, solve the double differential phase equations and
get the base-line Ionospheric delays which are frequency dependent
and the Tropospheric delay which are frequency independent.

However, the LAMBDA method does not always find the optimal integer
solutions. We use Closed Loop and n-sigma outlier tests to check if
the solved base-lines are self-consistent as shown in Figure \ref{fig:gnss_background}(c).
\textbf{We mark the data point as clean if it passes the tests otherwise
we mark it as noisy/corrupted.}\emph{ }The two tests work well when
the ionosphere is quiet. In summer or in equatorial region, ionosphere
fluctuates violently and the LAMBDA method does not work well. The
low quality of integer solutions found by LAMBDA drives the Closed
Loop and n-sigma test to over-kill and under-kill data, resulting
in large noise and missing data.

\subsection{Filter-Reweight-Retrain Algorithm For GPS Signal Enhancement \label{subsec:Filter-Reweight-Retrain-Meta-Alg}}

\begin{algorithm}
\begin{algorithmic}[1]

\REQUIRE Clean set $\mathcal{S}$, noisy set $\hat{\mathcal{S}}$,
truncation threshold $\tau$and reweighting function $h(\cdot)$ in
Eq. (\ref{eq:practice-choice-of-tau-alpha}).

\ENSURE Robust estimation $\bar{\boldsymbol{w}}$.

\STATE $\boldsymbol{w}^{\mathrm{init}}=\underset{\boldsymbol{w}}{\mathrm{argmin}}\ \ell(\mathcal{S}|\boldsymbol{w})$.

\STATE $r_{i}=\left|\left\langle \boldsymbol{w}^{\mathrm{init}},\hat{\boldsymbol{x}}^{(i)}\right\rangle -\hat{y}_{i}\right|$,
$\forall(\hat{\boldsymbol{x}}^{(i)},\hat{y}_{i})\in\hat{S}$.

\STATE $\hat{\mathcal{S}}'=\{(\hat{\boldsymbol{x}}^{(i)},\hat{y}_{i})|\ r_{i}\leq\tau,\ \forall(\hat{\boldsymbol{x}}^{(i)},\hat{y}_{i})\in\hat{S}$.

\STATE $\alpha_{i}=h(r_{i})$ if $\forall\hat{\boldsymbol{x}}^{(i)}\in\hat{\mathcal{S}}'$
otherwise $\alpha_{i}=1$.

\STATE Retrain weighted least square
\[
\bar{\boldsymbol{w}}=\underset{\boldsymbol{w}}{\mathrm{argmin}}\sum_{(\boldsymbol{x}^{(i)},y_{i})\in\mathcal{S}\cup\hat{\mathcal{S}}'}\alpha_{i}\left(\left\langle \boldsymbol{w},\boldsymbol{x}^{(i)}\right\rangle -y_{i}\right)^{2}\ .
\]

\STATE \textbf{Output: $\bar{\boldsymbol{w}}$.}

\end{algorithmic}

\caption{Filter-Reweight-Retrain Meta Algorithm}

\label{alg:FRR}
\end{algorithm}

The FRR is applicable to any base estimator that can generate decision
value. For sake of simplicity, we restrict ourselves in $d$ dimensional
linear space. Following the standard settings, suppose that the data
point $\boldsymbol{x}\in\mathbb{R}^{d}$ are sampled from sub-gaussian
distribution such that
\[
\mathbb{E}\boldsymbol{x}=\boldsymbol{0},\ \|\boldsymbol{x}\|_{2}\leq c\sqrt{d\log(2d/\eta)},\ \mathbb{E}\boldsymbol{x}\boldsymbol{x}{}^{\top}=I
\]
 where $c$ is a universal constant, $0<\eta<1$. The label $y\in\mathbb{R}$
is a real number. We are given two datasets consisting of $(\boldsymbol{x},y)$
pairs, the clean set $\mathcal{S}$ and the noisy set $\hat{\mathcal{S}}$.
Both sets are independently and identically (i.i.d.) sampled. In the
clean set $\mathcal{S}$, we assume that
\[
y=\left\langle \boldsymbol{w}^{*},\boldsymbol{x}\right\rangle +\xi\quad\forall(\boldsymbol{x},y)\in\mathcal{S}\ .
\]
 The noise term $\xi$ is i.i.d sampled from sub-gaussian distribution
with mean zero, variance proxy $\sigma$ and boundary $|\xi|\leq\xi_{\max}$.
In the noisy set $\hat{\mathcal{S}}$, the observed label is corrupted
by the noise model
\[
\hat{y}=\left\langle \boldsymbol{w}^{*},\hat{\boldsymbol{x}}\right\rangle +\hat{\xi}\quad\forall(\hat{\boldsymbol{x}},\hat{y})\in\hat{\mathcal{S}}\ .
\]
 The noise term $\hat{\xi}$ is i.i.d. sampled from heavy-tailed distribution.
It is important to emphasize that even the mean value of $\hat{\xi}$
may not exist therefore directly learning on $\hat{\mathcal{S}}$
is impossible. Without loss of generality, we assume that the distribution
of $\hat{\xi}$ is symmetric. Otherwise we add a bias term in the
linear regression model to capture the offset. The volume of the clean
set $\mathcal{S}$ is denoted as $n=|\mathcal{S}|$ and the volume
of the noisy set $m=|\hat{\mathcal{S}}|$ . $|\hat{\mathcal{S}}|/(|\mathcal{S}|+|\hat{\mathcal{S}}|)$
is the noise ratio of the full dataset. We choose linear least square
regression as our base estimator defined by
\[
\widehat{\boldsymbol{w}}=\underset{\boldsymbol{w}}{\mathrm{argmin}}\ \ell(\mathcal{S}|\boldsymbol{w})\triangleq\sum_{(\boldsymbol{x},y)\in\mathcal{S}}(\left\langle \boldsymbol{w},\boldsymbol{x}\right\rangle -y)^{2}\ .
\]
 The recovery error $\|\widehat{\boldsymbol{w}}-\boldsymbol{w}^{*}\|_{2}$
is used as performance index in our theoretical analysis.

Our FRR algorithm is detailed in Algorithm \ref{alg:FRR}. First we
train an initial estimator with the clean dataset $\mathcal{S}$.
The residual error $r_{i}$ is computed for every data point in the
noisy dataset $\hat{\mathcal{S}}$ and then any $\hat{\boldsymbol{x}}\in\hat{\mathcal{S}}$
with residual error larger than the threshold $\tau$ is filtered
out. The noisy dataset after filtering is denoted as $\hat{\mathcal{S}}'$.
We assign instance weight $\alpha_{i}$ for each $\hat{\boldsymbol{x}}^{(i)}\in\hat{\mathcal{S}}'$.
Finally we retrain the weighted least square to get the robust estimation
$\bar{\boldsymbol{w}}$.

The reader might note that we have not specify how to choose $\tau$
and $\alpha_{i}$ in Algorithm \ref{alg:FRR}. It is possible to adaptively
choose these parameters such that the recovery error bound given in
Theorem \ref{thm:recovery-error-retraining} (Section \ref{sec:Theoretical-Analysis})
is minimized. . When the volume of the noisy set is infinite, we can
simply choose $\tau$ sufficiently large and $\alpha_{i}=1$. Theorem
\ref{thm:recovery-error-retraining} guarantees that the recovery
error is asymptotically zero in this case. However when the volume
of the noisy set is finite, finding the optimal parameters requires
more efforts. In practice we suggest to choose
\begin{equation}
\tau=c_{1}\sum_{i}r_{i}/m,\ \alpha_{i}=c_{2}\exp(-r_{i}/c_{3})\label{eq:practice-choice-of-tau-alpha}
\end{equation}
 where $\{c_{1},c_{2},c_{3}\}$ are tuned by cross-validation. The
reader is free to design any filtering and reweighting scheme if he/she
has more prior knowledge about the noise distribution. However, several
rules must be followed when designing the filtering and reweighting
scheme. Please check the discussion under Theorem \ref{thm:recovery-error-retraining}
to see why those rules are important to the success of FRR.

\paragraph{Remark 1}

Although Algorithm \ref{alg:FRR} only concerns about linear estimator
and least square loss function, the FRR is essentially a meta algorithm
applicable to non-linear kernel machines and general convex loss function,
for example Gaussian process regression. The feature vector $\boldsymbol{x}$
could be replaced by $\psi(\boldsymbol{x})$ where $\psi(\cdot)$
is the feature map function induced by the kernel function and the
inner product $\left\langle \boldsymbol{w},\boldsymbol{x}\right\rangle $
could be replaced by the corresponding inner product defined on the
Reproducing Hilbert Kernel Space (RHKS). However in this generalized
formulation, our proof should be modified respectively. For example
we should replace matrix concentration with Talagrand's concentration.
To keep things intuitive without loss of generality, we will focus
on the linear model in our theoretical analysis.

\section{Theoretical Analysis \label{sec:Theoretical-Analysis}}

In this section, we give the statistical learning guarantees of FRR.
Our main result is given in Theorem \ref{thm:recovery-error-retraining}
which claims the recovery error bound. Two non-trivial rules for effective
filtering and reweighting are derived from Theorem \ref{thm:recovery-error-retraining}.

First we show that the initial estimator trained on clean dataset
is not far away from the ground-truth $\boldsymbol{w}^{*}$. This
is a well-known result in linear regression. The proof is given in
Appendix \ref{sec:Proof-of-Lemma-init-estimator-is-good-enough}.
\begin{lem}
\label{lem:init-estimator-is-good-enough} In Algorithm \ref{alg:FRR},
with probability at least $1-\eta$,
\[
\|\boldsymbol{w}^{\mathrm{init}}-\boldsymbol{w}^{*}\|_{2}\leq c^{2}\sigma\sqrt{d}[\log(2d/\eta)]^{3/2}/\sqrt{|\mathcal{S}|}\triangleq\Delta
\]
 provided that $|\mathcal{S}|\geq\max\{(\xi_{\max}/\sigma)^{2},4d\log^{2}(2d/\eta)/c^{6}\}$.
\end{lem}

Lemma \ref{lem:init-estimator-is-good-enough} shows that the initial
estimator will not be far away from $\boldsymbol{w}^{*}$ as long
as the clean dataset is sufficiently large. The FRR only requires
the volume of the clean dataset $|\mathcal{S}|$ no less than $\mathcal{O}(d\log d)$.
The major challenge is how to retrain on the noisy dataset. Recall
that $\hat{\xi}$ is heavy-tailed so we cannot bound any concentration
directly on $\hat{\mathcal{S}}$. Before we can do any learning on
the noisy dataset, we must bound the variance of $\hat{\xi}$. To
this end, the FRR uses $\boldsymbol{w}^{\mathrm{init}}$ to truncate
$\hat{\xi}$ via instance filtering. After the filtering step on line
3 in Algorithm \ref{alg:FRR}, for any $\hat{\boldsymbol{x}}^{(i)}\in\hat{\mathcal{S}}'$
we have the following lemma. The proof is given in Appendix \ref{sec:Proof-of-Lemma-trunc-noise}.
\begin{lem}
\label{lem:trunc-noise-level} In Algorithm \ref{alg:FRR}, with probability
at least $1-2\eta$, we have
\[
|\hat{\xi_{i}}|\leq r_{i}+c\Delta\sqrt{\log(2/\eta)}\quad\forall\hat{\boldsymbol{x}}^{(i)}\in\hat{\mathcal{S}}'\ .
\]
\end{lem}

Lemma \ref{lem:trunc-noise-level} shows that when we take $\tau\geq c\Delta\sqrt{\log(2/\eta)}$,
we can truncate the noise level on $\hat{\mathcal{S}}'$ below $2\tau$.
It is important to note that in the proof of Lemma \ref{lem:trunc-noise-level},
we use the independence of $\boldsymbol{w}^{\mathrm{init}}$ and $\hat{\boldsymbol{x}}^{(i)}\in\hat{\mathcal{S}}'$
. This indicates that one cannot retrain on $\hat{\mathcal{S}}$ twice.
Actually in the second round of retraining, since previous $\boldsymbol{w}$
depends on $\boldsymbol{x}^{(i)}$, the bound $|\boldsymbol{w}{}^{\top}\boldsymbol{x}^{(i)}|$
will degrade to $\mathcal{O}(d/\sqrt{|\mathcal{S}|})$ which requires
$|\mathcal{S}|\geq\mathcal{O}(d^{2})$. However, if the clean dataset
is larger than $\mathcal{O}(d^{2})$, the initial estimator is already
good enough and no retraining is necessary at all. In the experiment,
we find that iteratively retraining indeed hurts the performance.

After filtering, we get a bounded noise distribution on $\hat{\mathcal{S}}'$
making learning possible. Suppose after filtering, we have $|\hat{\mathcal{S}}'|=m$.
The sample weights $\boldsymbol{\alpha}=[\alpha_{1},\alpha_{2},\cdots,\alpha_{m}]$
are computed via $h(r_{i})$. The following main theorem bounds the
recovery error after retraining.
\begin{thm}
\label{thm:recovery-error-retraining} In Algorithm \ref{thm:recovery-error-retraining},
denote $|\mathcal{S}|=n$, $|\hat{\mathcal{S}}'|=m$. $\{\alpha_{1},\alpha_{2},\dots,\alpha_{m}\}$
are weights of elements in $\hat{\mathcal{S}}'$, $\alpha_{\min}\leq\alpha_{i}\leq\alpha_{\max}$.
$\boldsymbol{x}^{(i)},\hat{\boldsymbol{x}}^{i)}$ are sub-gaussian
random vectors with bounded norm $\|\boldsymbol{x}^{(i)}\|_{2},\|\hat{\boldsymbol{x}}^{(i)}\|_{2}\leq c\sqrt{d\log(2d/\eta)}$
where $c>0$ is a universal constant, $0<\eta<1$. Define
\begin{align*}
\bar{\sigma}_{1}^{2} & \triangleq\|\mathbb{E}\hat{\xi}_{i}^{2}\alpha_{i}^{2}\hat{\boldsymbol{x}}^{(i)}\hat{\boldsymbol{x}}^{(i)}{}^{\top}\|_{2}\\
\bar{\alpha} & \triangleq\mathbb{E}\{\|\alpha_{i}\hat{\boldsymbol{x}}^{(i)}\hat{\boldsymbol{x}}^{(i)}{}^{\top}\|_{2}\}\ .
\end{align*}
 Choose $\tau$ such that $\{\bar{\sigma}_{1},\bar{\alpha}\}$ are
bounded and $m\geq1$. Then with probability at least $1-3\eta$,
\begin{multline}
\|\bar{\boldsymbol{w}}-\boldsymbol{w}^{*}\|_{2}\leq[\frac{1}{2}n+m\alpha_{\min}-c^{3}\bar{\alpha}\sqrt{md}\log(2d/\eta)]^{-1}\\
\{c^{2}\sigma\sqrt{nd}[\log(2d/\eta)]^{3/2}+\bar{\sigma}_{1}\sqrt{m\log(2d/\eta)}\}\label{eq:ffr-recovery-bound}
\end{multline}
 provided
\begin{align*}
n\geq & \max\{(\xi_{\max}/\sigma)^{2},4d\log^{2}(2d/\eta)/c^{6}\}\\
m\geq & d\log^{2}(2d/\eta)\max\{[\alpha_{\max}(\tau+c\Delta\sqrt{\log(2/\eta)})/\bar{\sigma}_{1}]^{2}\\
 & c^{6}(\bar{\alpha}/\alpha_{\min})^{2},(\alpha_{\max}/\bar{\alpha})^{2}\}\ .
\end{align*}
\end{thm}

\begin{proof}
Denote $X=[\boldsymbol{x}^{(1)},\boldsymbol{x}^{(2)},\cdots,\boldsymbol{x}^{(n)}]$,
$\hat{X}=[\hat{\boldsymbol{x}}^{(1)},\hat{\boldsymbol{x}}^{(2)},\cdots,\hat{\boldsymbol{x}}^{(n)}]$,
$\boldsymbol{y}=[y_{1},y_{2},\cdots,y_{n}]$, $\hat{\boldsymbol{y}}=[\hat{y}_{1},\hat{y}_{2},\cdots,\hat{y}_{m}]$,
$\boldsymbol{\alpha}=[\alpha_{1},\alpha_{2},\cdots,\alpha_{m}]$.
$\mathcal{D}(\cdot)$ is the diagonal function.

According to Algorithm \ref{alg:FRR},
\begin{align*}
\bar{\boldsymbol{w}}= & \underset{\boldsymbol{w}}{\mathrm{argmin}}\ \mathcal{L}(\boldsymbol{w})\\
\triangleq & \|X{}^{\top}\boldsymbol{w}-y\|^{2}+\mathrm{tr}\{(\hat{X}{}^{\top}\boldsymbol{w}-\hat{\boldsymbol{y}}){}^{\top}\mathcal{D}(\boldsymbol{\alpha})(\hat{X}{}^{\top}\boldsymbol{w}-\hat{\boldsymbol{y}})\}\ .
\end{align*}
 The derivation is
\begin{align*}
\nabla_{\boldsymbol{w}}\mathcal{L}(\boldsymbol{w}) & =XX^{\top}\boldsymbol{w}-X\boldsymbol{y}+\hat{X}\mathcal{D}(\boldsymbol{\alpha})\hat{X}{}^{\top}\boldsymbol{w}-\hat{X}\mathcal{D}(\boldsymbol{\alpha})\hat{\boldsymbol{y}}\\
 & =(XX^{\top}+\hat{X}\mathcal{D}(\boldsymbol{\alpha})\hat{X}{}^{\top})\boldsymbol{w}-(X\boldsymbol{y}+\hat{X}\mathcal{D}(\boldsymbol{\alpha})\hat{\boldsymbol{y}})\ .
\end{align*}
 Similar to Lemma \ref{lem:init-estimator-is-good-enough},
\begin{align*}
\bar{\boldsymbol{w}}= & (XX^{\top}+\hat{X}\mathcal{D}(\boldsymbol{\alpha})\hat{X}{}^{\top})^{-1}(X\boldsymbol{y}+\hat{X}\mathcal{D}(\boldsymbol{\alpha})\hat{\boldsymbol{y}})\\
= & (XX^{\top}+\hat{X}\mathcal{D}(\boldsymbol{\alpha})\hat{X}{}^{\top})^{-1}\\
 & (XX^{\top}\boldsymbol{w}^{*}+X\boldsymbol{\xi}+\hat{X}\mathcal{D}(\boldsymbol{\alpha})X^{\top}\boldsymbol{w}^{*}+\hat{X}\mathcal{D}(\boldsymbol{\alpha})\boldsymbol{\xi})\\
= & (XX^{\top}+\hat{X}\mathcal{D}(\boldsymbol{\alpha})\hat{X}{}^{\top})^{-1}\\
 & [(XX^{\top}+\hat{X}\mathcal{D}(\boldsymbol{\alpha})X^{\top})\boldsymbol{w}^{*}+X\boldsymbol{\xi}+\hat{X}\mathcal{D}(\boldsymbol{\alpha})\hat{\boldsymbol{\xi}}]\\
= & \boldsymbol{w}^{*}+(XX^{\top}+\hat{X}\mathcal{D}(\boldsymbol{\alpha})\hat{X}{}^{\top})^{-1}(X\boldsymbol{\xi}+\hat{X}\mathcal{D}(\boldsymbol{\alpha})\hat{\boldsymbol{\xi}})\ .
\end{align*}

To ensure the matrix inversion exists, we need to ensure that $XX^{\top}$
is invertible as $\hat{X}\mathcal{D}(\boldsymbol{\alpha})\hat{X}{}^{\top}$
is a symmetric positive semi-definite matrix. According to Lemma \ref{lem:init-estimator-is-good-enough},
when
\[
n\geq4d\log^{2}(2d/\eta)/c^{6}\ ,
\]
 we have with probability at least $1-\eta$,
\[
\lambda_{\mathrm{min}}\{\frac{1}{n}XX^{\top}\}\geq1/2
\]
 where $\lambda_{\mathrm{min}}\{\cdot\}$ is the smallest eigenvalue.

To bound $\lambda_{\mathrm{min}}\{\hat{X}\mathcal{D}(\boldsymbol{\alpha})\hat{X}{}^{\top}\}$,
\[
\hat{X}\mathcal{D}(\boldsymbol{\alpha})\hat{X}{}^{\top}=\sum_{i=1}^{m}\alpha_{i}\hat{\boldsymbol{x}}^{(i)}\hat{\boldsymbol{x}}^{(i)}{}^{\top}\ ,
\]
\begin{align*}
\mathbb{E}\{\hat{X}\mathcal{D}(\boldsymbol{\alpha})\hat{X}{}^{\top}\}= & \sum_{i=1}^{m}\mathbb{E}\{\alpha_{i}\hat{\boldsymbol{x}}^{(i)}\hat{\boldsymbol{x}}^{(i)}{}^{\top}\}\ .
\end{align*}
 Please note that $\alpha_{i}$ and $\hat{\boldsymbol{x}}^{(i)}$
are not independent so we cannot simply write $\mathbb{E}\{\alpha_{i}\hat{\boldsymbol{x}}^{(i)}\hat{\boldsymbol{x}}^{(i)}{}^{\top}\}=\alpha_{i}I$.
As $\alpha_{i}\hat{\boldsymbol{x}}^{(i)}\hat{\boldsymbol{x}}^{(i)}{}^{\top}$
are symmetric PSD matrices,
\begin{align*}
\lambda_{\mathrm{min}}\{\mathbb{E}[\hat{X}\mathcal{D}(\boldsymbol{\alpha})\hat{X}{}^{\top}]\}\geq & m\alpha_{\min}\ .
\end{align*}
 Next we bound the concentration of $\hat{X}\mathcal{D}(\boldsymbol{\alpha})\hat{X}{}^{\top}$.
Applying matrix Bernstein's inequality, with probability at least
$1-\eta$,
\begin{align*}
 & \|\hat{X}\mathcal{D}(\boldsymbol{\alpha})\hat{X}{}^{\top}-\mathbb{E}\{\hat{X}\mathcal{D}(\boldsymbol{\alpha})\hat{X}{}^{\top}\}\|_{2}\\
\leq & c\max\{c^{2}\alpha_{\max}d\log^{2}(2d/\eta),c^{2}\bar{\alpha}\sqrt{md}\log(2d/\eta)\}.
\end{align*}
 Suppose
\begin{align*}
 & c^{2}\bar{\alpha}\sqrt{md}\log(2d/\eta)\geq c^{2}\alpha_{\max}d\log^{2}(2d/\eta)\\
\Leftarrow & m\geq(\alpha_{\max}/\bar{\alpha})^{2}d\log^{2}(2d/\eta)
\end{align*}
 Then
\begin{align*}
 & \|\hat{X}\mathcal{D}(\boldsymbol{\alpha})\hat{X}{}^{\top}-\mathbb{E}\{\hat{X}\mathcal{D}(\boldsymbol{\alpha})\hat{X}{}^{\top}\}\|_{2}\\
\leq & c^{3}\bar{\alpha}\sqrt{md}\log(2d/\eta)\ .
\end{align*}
Therefore with probability at least $1-\eta$ :
\begin{align*}
\lambda_{\min}\{\hat{X}\mathcal{D}(\boldsymbol{\alpha})\hat{X}{}^{\top}\}\geq & m\alpha_{\min}-c^{3}\bar{\alpha}\sqrt{md}\log(2d/\eta)
\end{align*}
 provided $m\geq(\alpha_{\max}/\bar{\alpha})^{2}d\log^{2}(2d/\eta)$
. To ensure the lower bound is meaningful, we must constrain
\begin{align*}
 & m\alpha_{\min}-c^{3}\bar{\alpha}\sqrt{md}\log(2d/\eta)\geq0\\
\Leftarrow & m\geq c^{6}(\bar{\alpha}/\alpha_{\min})^{2}d\log^{2}(2d/\eta)\ .
\end{align*}

Combining all above together, we have with probability at least $1-\eta$,
\[
\lambda_{\min}\{XX^{\top}+\hat{X}\mathcal{D}(\boldsymbol{\alpha})\hat{X}{}^{\top}\}\geq\frac{1}{2}n+m\alpha_{\min}-c^{3}\bar{\alpha}\sqrt{md}\log(2d/\eta)
\]
 provided
\[
m\geq\max\{c^{6}(\bar{\alpha}/\alpha_{\min})^{2}d\log^{2}(2d/\eta),(\alpha_{\max}/\bar{\alpha})^{2}d\log(2d/\eta)\}\ .
\]

To bound $\|X\boldsymbol{\xi}\|_{2}$ , according to Lemma \ref{lem:init-estimator-is-good-enough},
with probability at least $1-\eta$,
\[
\|X\boldsymbol{\xi}\|_{2}\leq c^{2}\sigma\sqrt{nd}[\log(2d/\eta)]^{3/2}
\]
 provided $n\geq(\xi_{\max}/\sigma)^{2}$.

To bound $\|\hat{X}\mathcal{D}(\boldsymbol{\alpha})\hat{\boldsymbol{\xi}}\|_{2}$,
\begin{align*}
\hat{X}\mathcal{D}(\boldsymbol{\alpha})\hat{\boldsymbol{\xi}}= & \sum_{i=1}^{m}\hat{\xi}_{i}\alpha_{i}\hat{\boldsymbol{x}}^{(i)}\ .
\end{align*}
According to the assumption, the noise distribution is assumed to
be symmetric,
\begin{align*}
\mathbb{E}\hat{X}\mathcal{D}(\boldsymbol{\alpha})\hat{\boldsymbol{\xi}}= & \boldsymbol{0}\ .
\end{align*}
 In order to apply matrix Bernstein's inequality again,
\begin{align*}
\|\hat{\xi}_{i}\alpha_{i}\hat{\boldsymbol{x}}^{(i)}\|_{2}\leq & [r_{i}+c\Delta\sqrt{\log(2/\eta)}]\alpha_{\max}\sqrt{d\log(2d/\eta)}\\
\leq & \alpha_{\max}[\tau+c\Delta\sqrt{\log(2/\eta)}]\sqrt{d\log(2d/\eta)}\ .
\end{align*}
  Therefore with probability at least $1-\eta$,
\begin{align*}
\|\hat{X}\mathcal{D}(\boldsymbol{\alpha})\hat{\boldsymbol{\xi}}\|_{2}\leq & c\max\{\alpha_{\max}[\tau+c\Delta\sqrt{\log(2/\eta)}]\sqrt{d}[\log(2d/\eta)]^{3/2}\\
 & ,\bar{\sigma}_{1}\sqrt{m\log(2d/\eta)}\}\ .
\end{align*}
 Suppose
\begin{align*}
 & \bar{\sigma}_{1}\sqrt{m\log(2d/\eta)}\geq\alpha_{\max}[\tau+c\Delta\sqrt{\log(2/\eta)}]\sqrt{d}[\log(2d/\eta)]^{3/2}\\
\Leftarrow & m\geq\{\alpha_{\max}[\tau+c\Delta\sqrt{\log(2/\eta)}]/\bar{\sigma}_{1}\}^{2}d\log^{2}(2d/\eta)\ .
\end{align*}
 We then get
\[
\|\hat{X}\mathcal{D}(\boldsymbol{\alpha})\hat{\boldsymbol{\xi}}\|_{2}\leq\bar{\sigma}_{1}\sqrt{m\log(2d/\eta)}\ .
\]

Combining all together, with probability at least $1-3\eta$,
\begin{align*}
\|\bar{\boldsymbol{w}}-\boldsymbol{w}^{*}\|_{2}\leq & \|(XX^{\top}+\hat{X}\mathcal{D}(\boldsymbol{\alpha})\hat{X}{}^{\top})^{-1}(X\boldsymbol{\xi}+\hat{X}\mathcal{D}(\boldsymbol{\alpha})\hat{\boldsymbol{\xi}})\|_{2}\\
\leq & \lambda_{\min}^{-1}\{XX^{\top}+\hat{X}\mathcal{D}(\boldsymbol{\alpha})\hat{X}{}^{\top}\}[\|X\boldsymbol{\xi}\|_{2}+\|\hat{X}\mathcal{D}(\boldsymbol{\alpha})\hat{\boldsymbol{\xi}}\|_{2}]\\
\leq & [\frac{1}{2}n+m\alpha_{\min}-c^{3}\bar{\alpha}\sqrt{md}\log(2d/\eta)]^{-1}\\
 & \{c^{2}\sigma\sqrt{nd}[\log(2d/\eta)]^{3/2}+\bar{\sigma}_{1}\sqrt{m\log(2d/\eta)}\}
\end{align*}
 provide
\begin{align*}
n\geq & \max\{(\xi_{\max}/\sigma)^{2},4d\log^{2}(2d/\eta)/c^{6}\}\\
m\geq & d\log^{2}(2d/\eta)\max\{[\alpha_{\max}(\tau+c\Delta\sqrt{\log(2/\eta)})/\bar{\sigma}_{1}]^{2}\\
 & c^{6}(\bar{\alpha}/\alpha_{\min})^{2},(\alpha_{\max}/\bar{\alpha})^{2}\}\ .
\end{align*}
\end{proof}
Theorem \ref{thm:recovery-error-retraining} is our main result. It
delivers several important messages. Roughly speaking, the recovery
error of FRR is bounded by $\mathcal{O}[(\sigma\sqrt{n}+\bar{\sigma}_{1}\sqrt{m})/(n+m\alpha_{\min})]$.
The choice of $\alpha_{i}=h(r_{i})$ is critical in our analysis.
If $\alpha_{i}$ depends on $r_{i}$, we cannot simply take $\mathbb{E}\alpha_{i}\hat{\boldsymbol{x}}^{(i)}\hat{\boldsymbol{x}}^{(i)}=\alpha_{i}I$
in the proof therefore there is no closed-form tight upper bound as
in Theorem \ref{thm:recovery-error-retraining}. The constants $\{\bar{\sigma}_{1}^{2},\bar{\alpha}\}$
are defined to capture such dependency. In the simplest scenario $\alpha_{i}=\alpha$
where $\alpha$ is a fixed constant independent from $\hat{\mathcal{S}}$,
the recovery error of FRR is then bounded by $\mathcal{O}[\sqrt{\sigma^{2}n+\alpha^{2}\hat{\sigma}^{2}m}/(n+\alpha m)]$
where $\hat{\sigma}^{2}$ is the variance of truncated noise in $\hat{\mathcal{S}}$.
Therefore for a fixed reweighting function $h(r_{i})=\alpha$, the
recovery error will always decrease when $m$ gets larger and $\hat{\sigma}$
controlled by $\tau$ increases at a lower rate. If we have prior
knowledge about the distribution of noise term $\hat{\xi}$, we could
choose $\tau$ to minimize the above error bound.

An obvious problem of the constant reweighting scheme is that it weights
instances of large noise and small noise equally. Observing that the
noise variance is multiplied by $\alpha_{i}^{2}$, it seems reasonable
to assign small weights for instances with large residual error $r_{i}$,
leading to a better (but actually incorrect) recovery error bound
$\mathcal{O}(\sqrt{\sigma^{2}n+\sum_{i}\alpha_{i}^{2}r_{i}^{2}}/(n+\sum_{i}\alpha_{i}))$.
We argue that the analysis of adaptive reweighting scheme is more
complex than the constant case. When $\alpha_{i}$ depends on $\hat{\mathcal{S}}$,
one cannot simply apply the expectation equality and concentration
inequality in the same way. A counter example is to select one instance
with $r_{i}\approx0$ and then set $\alpha_{i}\rightarrow\infty$,
$\alpha_{j}=0$ for $j\not=i$. This is equal to retrain the model
using the $i$ -th instance along. Clearly the recovery error bound
is not bounded by $r_{i}\approx0$. This counter example shows that
there are some rules we must follow when designing adaptive reweighting
scheme. Based on the proof of Theorem \ref{thm:recovery-error-retraining},
we summarize the following two rules:
\begin{itemize}
\item The reweighting function $h(\cdot)$ should depend on $r_{i}$ only.
That is, $\alpha_{i}=h(r_{i})$ must not refer to any other $\{r_{j}|j\not=i\}$.
This rule ensures that $\{\alpha_{i}\}_{i=1}^{m}$ are independent
to each other so that $\sum_{i}\alpha_{i}\hat{\boldsymbol{x}}^{(i)}\hat{\boldsymbol{x}}^{(i)}$
is concentrated around its mean value. Particularly, this rule prohibits
jointly and iteratively optimizing $\{\alpha_{i}\}_{i=1}^{m}$.
\item The maximum of $h(\cdot)$ should be bounded. This rule prevents the
reweighting step assigns large weight for one single instance.
\end{itemize}
Please be advised that the above two rules are derived from the concentration
inequalities we used in our proof. They are imposed by the nature
of our learning problem rather than the weakness of FRR. Informally
speaking, it is discouraged to tune $\alpha_{i}$ too aggressively.
In practice, we suggest to choose a base weight $c_{2}$ for all instances
in $\hat{\mathcal{S}}'$ and then decay the weight a bit proportional
to $r_{i}$. We find that $\alpha_{i}=c_{2}\exp(-r_{i}/c_{3})$ with
$\tau=c_{1}\sum_{i}r_{i}/m$ is a good choice for our problem. It
is easy to verify that this choice satisfies the above two rules.

\section{Experiment}

\begin{table}
\caption{Dataset Statistics}
 \label{tab:dataset_stat}
\begin{centering}
\begin{tabular}{llll}
\hline
Dataset & \# clean data & \# noisy data & noisy ratio\tabularnewline
\hline
TrainDay1 & 17751091 & 4288037 & 19.5\%\tabularnewline
TrainDay2 & 17443691 & 4619126 & 20.9\%\tabularnewline
TrainDay3 & 16198561 & 5851525 & 26.5\%\tabularnewline
TrainDay4 & 17091968 & 5042280 & 22.8\%\tabularnewline
TrainDay5 & 15807494 & 6305344 & 28.5\%\tabularnewline
TestDay1 & 16292229 & 5605534 & 25.6\%\tabularnewline
TestDay2 & 15048558 & 6983814 & 31.7\%\tabularnewline
TestDay3 & 14616549 & 7327788 & 33.4\%\tabularnewline
TestDay4 & 15191055 & 6849075 & 31.1\%\tabularnewline
TestDay5 & 13616039 & 8374171 & 38.1\%\tabularnewline
\hline
\end{tabular}
\par\end{centering}
\end{table}

\begin{figure*}
\begin{minipage}[b][][b]{0.32 \linewidth}
\includegraphics[width=\linewidth]{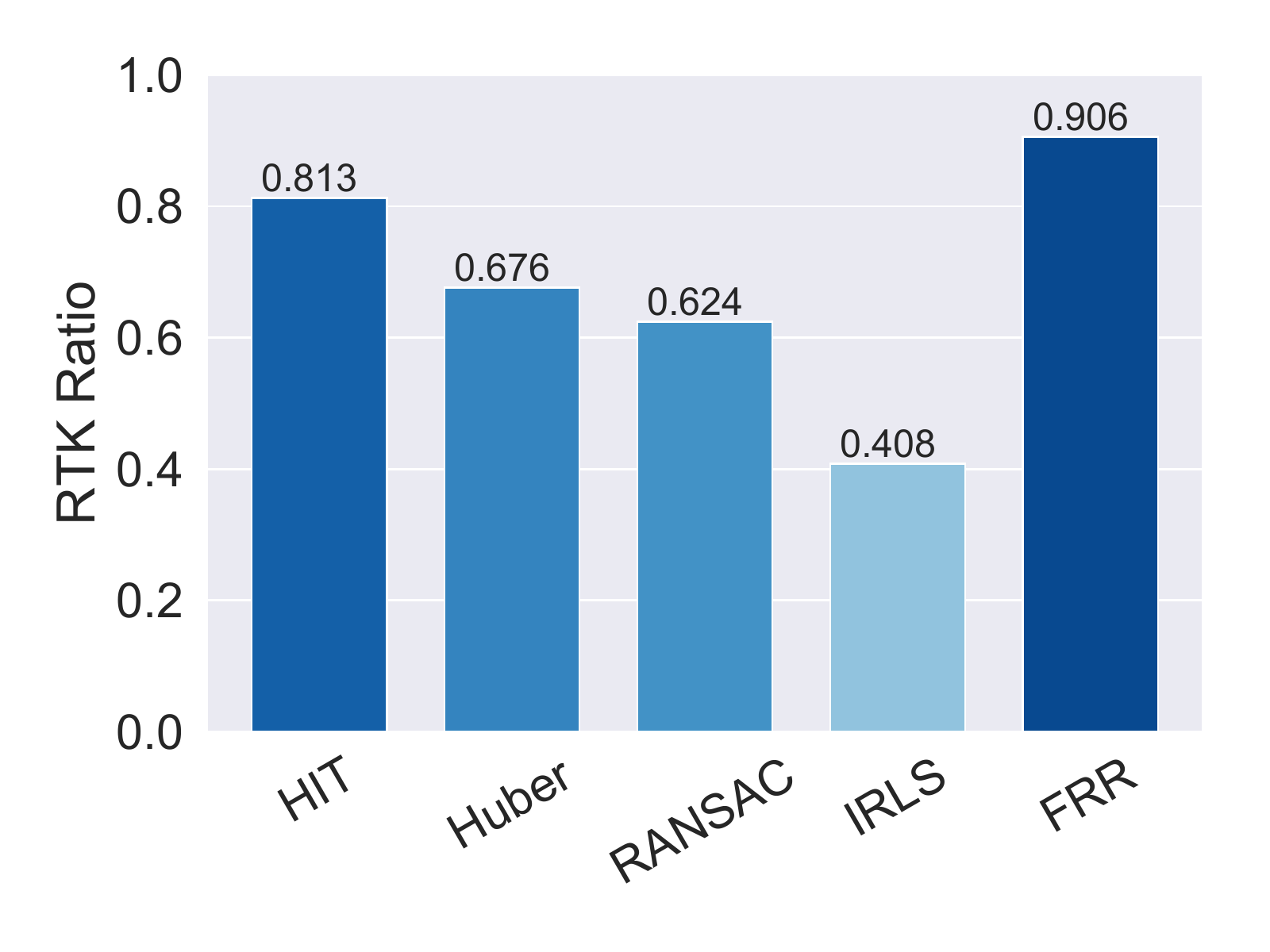}\\
\begin{center}
(a) TestDay1
\end{center}
\end{minipage} \hfil
\begin{minipage}[b][][b]{0.32 \linewidth}
\includegraphics[width=\linewidth]{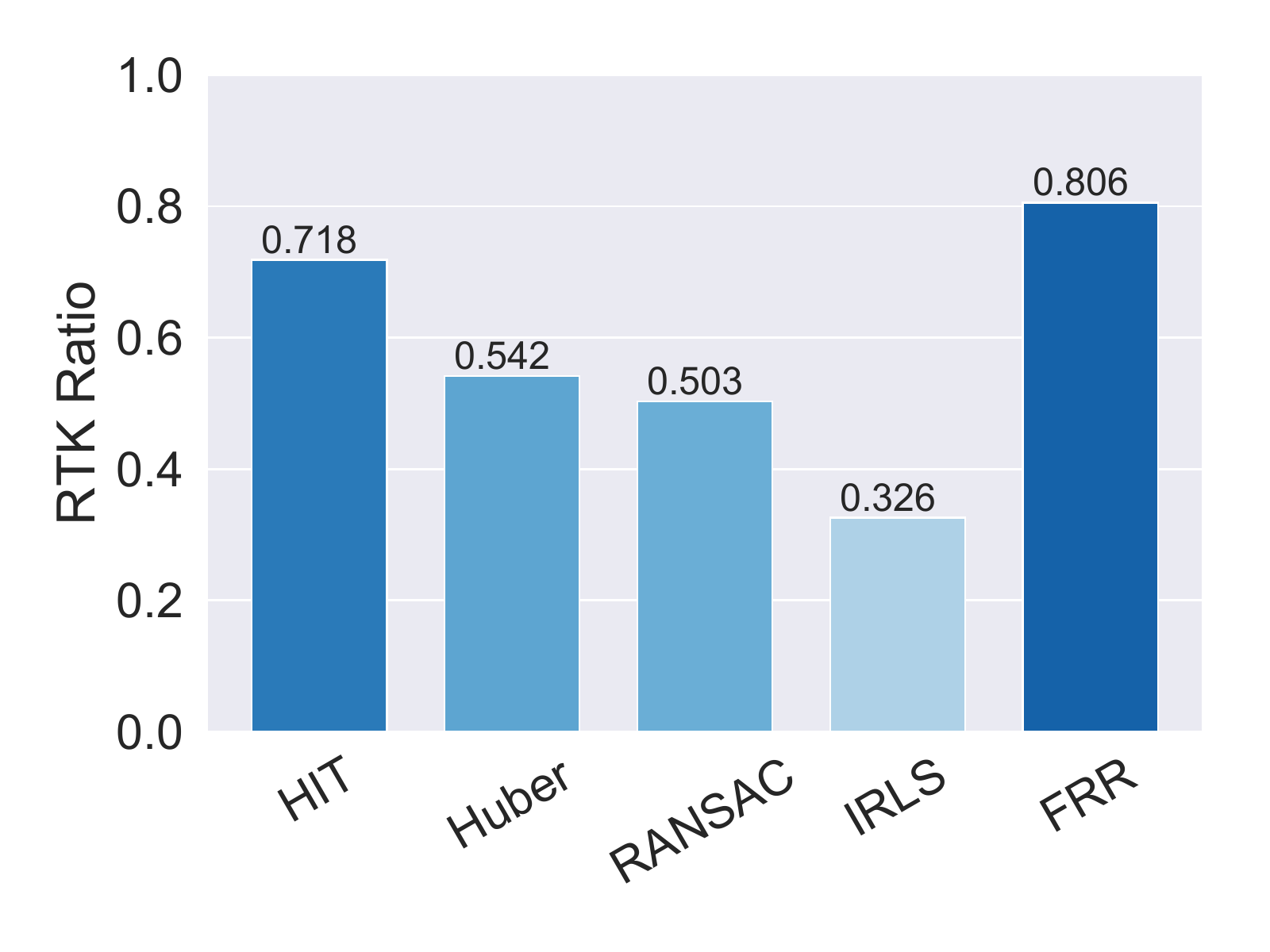}\\
\begin{center}
(b) TestDay2
\end{center}
\end{minipage} \hfil
\begin{minipage}[b][][b]{0.32 \linewidth}
\includegraphics[width=\linewidth]{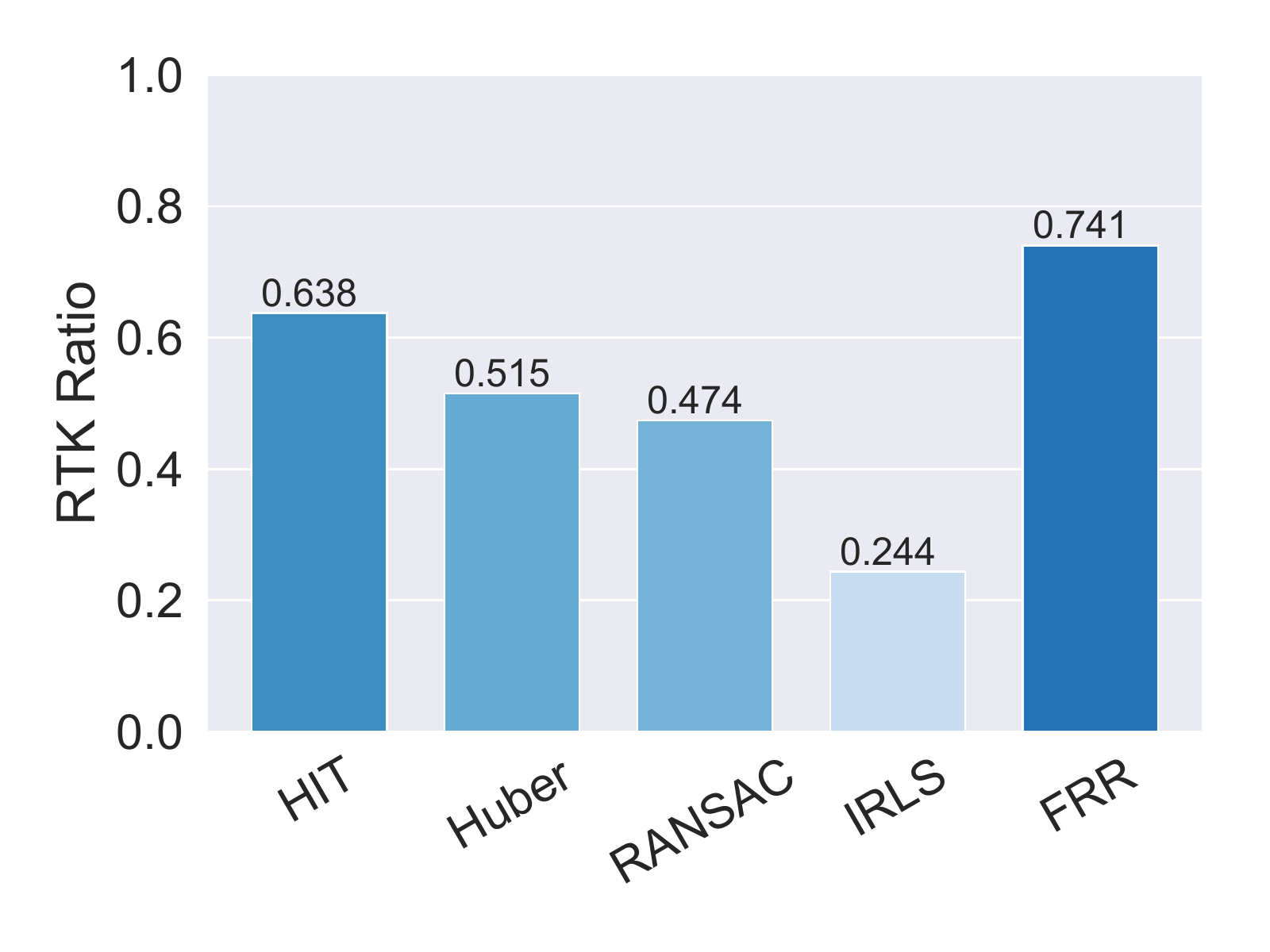}\\
\begin{center}
(c) TestDay3
\end{center}
\end{minipage} \\
\begin{minipage}[b][][b]{0.32 \linewidth}
\includegraphics[width=\linewidth]{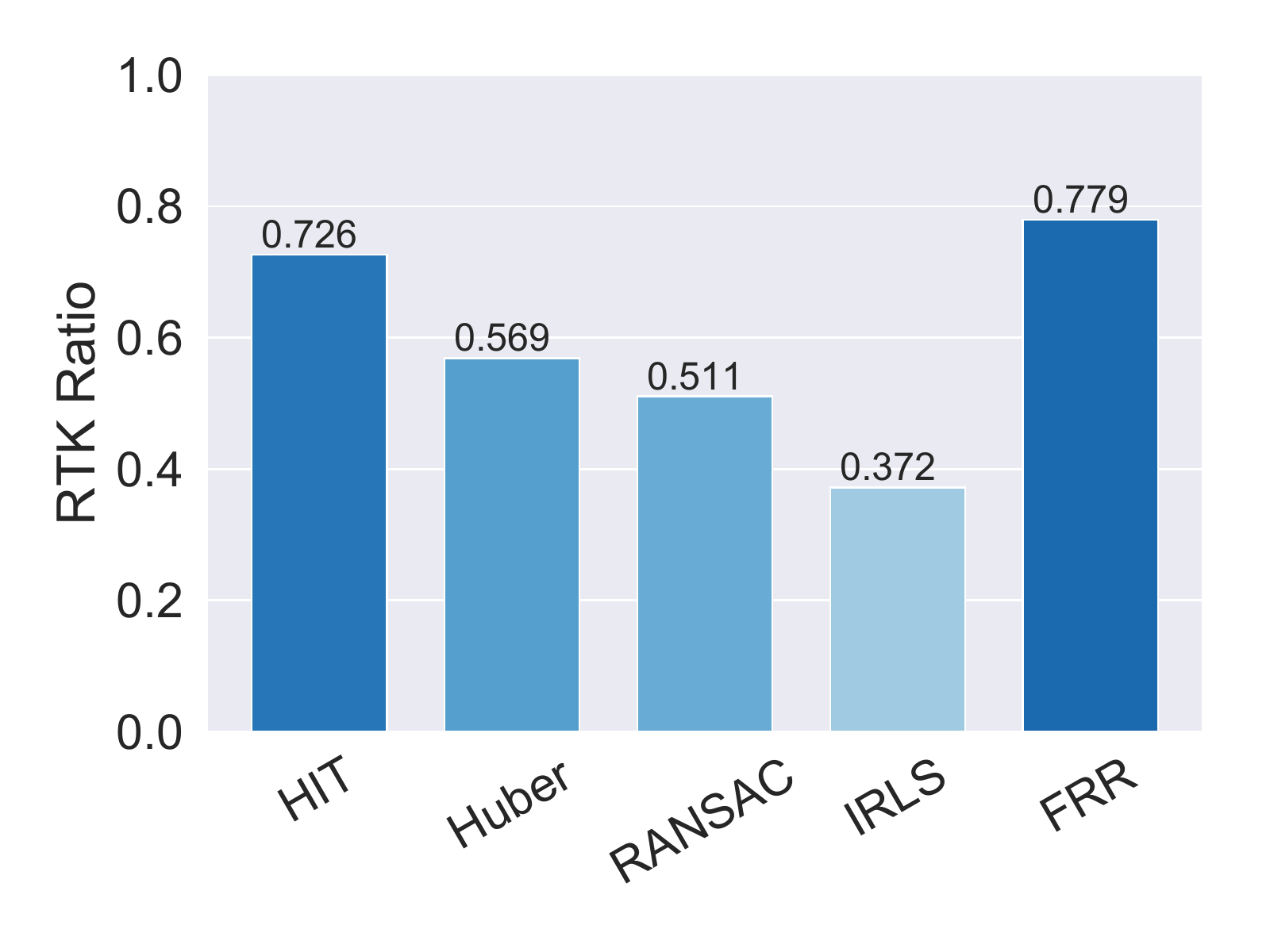}\\
\begin{center}
(d) TestDay4
\end{center}
\end{minipage} \hfil
\begin{minipage}[b][][b]{0.32 \linewidth}
\includegraphics[width=\linewidth]{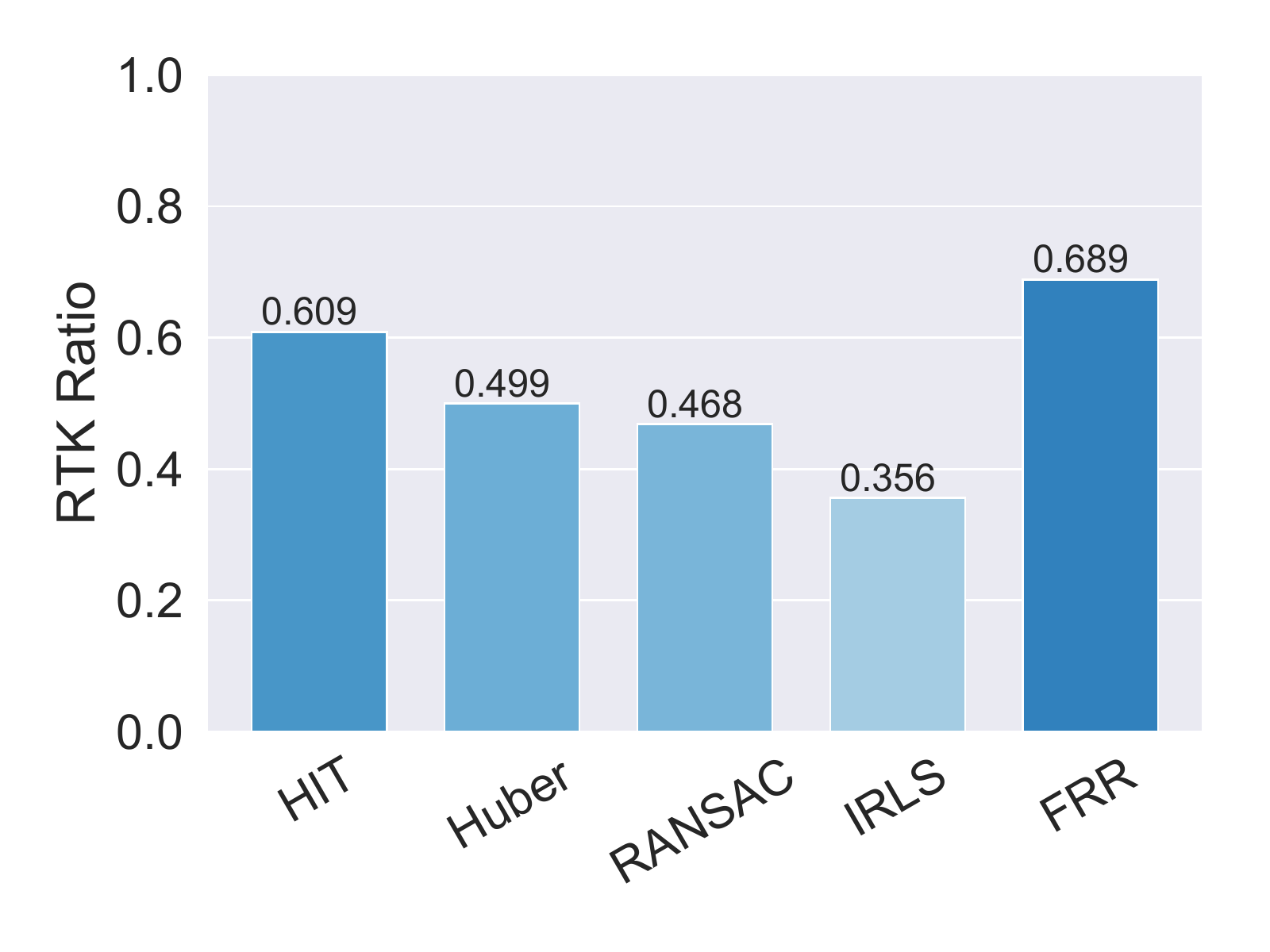}\\
\begin{center}
(e) TestDay5
\end{center}
\end{minipage} \hfil
\begin{minipage}[b][][b]{0.32 \linewidth}
\includegraphics[width=\linewidth]{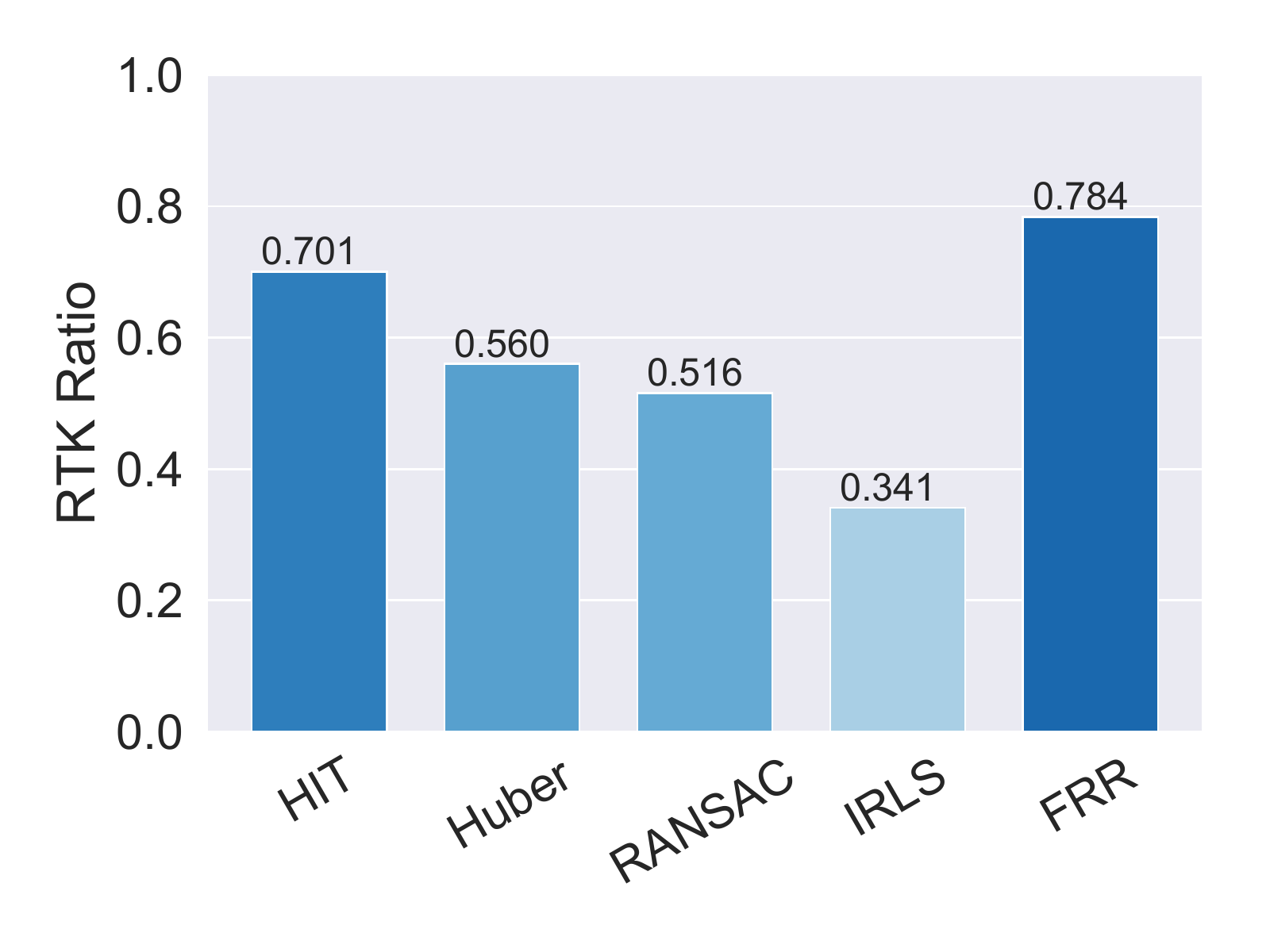}\\
\begin{center}
(f) Average
\end{center}
\end{minipage}\caption{RTK Ratio Comparison}
\label{fig:rtk_ratio}
\end{figure*}

We are using a Network-based Real-time Kinematic (RTK) carrier-phase
ground-based augmentation system to achieve centimeter-level positioning
service. A triangular mesh grid of 20 satellite stations spaced by
75 kilometers are distributed over 100 000 square kilometers area
in an anonymous region. They can receive and process signals of the
satellites from all 4 major GNSS systems. All stations are equipped
with dual-frequency receivers. When satellites' orbit above horizon
with elevation angle larger than 15 degrees, each receiver-satellite
pair forms a Pierce Point on the ionosphere. Their data will be sent
to a central processing facility hosted on the cloud to fit the model
for vTEC pseudorange compensation.

We collected a dataset consisting of 10 consecutive days for experiment.
We use the first 5 days as training set and the last 5 days as testing
set.  Due to data anonymity, we mark the training dates from TrainDay1
to TrainDay5 and the testing dates from TestDay1 to TestDay5. The
dataset statistics are summarized in Table \ref{tab:dataset_stat}.
The second and the third columns are the number of clean and noisy
data points respectively (See Subsection \ref{subsec:How-GPS-positioning}
for our algorithm to detect the noisy data points). The last column
reports the ratio of noisy data points with respect to the total number
of data points collected in one day. As shown in Table \ref{tab:dataset_stat},
around 30\% data points will be filtered out by our noisy data screening
algorithm. This is a great waste if no robust method is used to recycle
the noisy data points. On the other hand, simply using all data points
is harmful as we will show shortly.

It is important to note that the definition of training/testing in
our problem does not follow the convention. We choose one ground station
as the testing station and use its measured signal as ground-truth.
Around this testing station, we choose 16 ground stations as training
stations. The experiment consists of two stages: the offline parameter
tuning stage and the online prediction stage. In the offline parameter
tuning stage, we use the training set to tune parameters of our model
and the best model parameter is selected. In the online prediction
stage, we apply the best model parameter to the testing set. In both
offline and online stages, the performance of the model is evaluated
as following. In every second, we retrieve new data points from our
grid system and train the model on the training stations. Then the
learned model is applied to the testing station to prediction the
double difference for any quadruplet pierce point related to the testing
station. To avoid over-fitting, the quadruplet pierce points containing
the testing station are excluded in the offline stage.

To evaluate performance, we use RTK ratio as index, the higher the
better. The RTK ratio is the probability of successful positioning
while the positioning error is less than one centimeter. The RTK ratio
largely depends on the double differences we predicted for the testing
station.

As our system is a real-time online system, we require that the total
time cost of model training and prediction cannot exceed 300 milliseconds
per second frame. This excludes many time consuming models such as
deep neural network. After benchmarking, we adopt Gaussian Process
(GP) as our base estimator as it achieves the best RTK ratio among
other popular models. In our GP model, the vTEC for each pierce point
is denoted as $\boldsymbol{f}_{i}$ where $i=1,2,\cdots n$. The matrix
$A$ encodes the double-difference quadruplet: each row of $A$ has
exact two ``1'' entries and two ``-1'' entries where the corresponding
ionosphere pierce point is involved in the double-difference calculation.
Each ionosphere pierce point is presented as a four dimensional vector
$\boldsymbol{x}=[x_{1},x_{2},x_{3},x_{4}]\in\mathbb{R}^{4}$. $\{x_{1},x_{2}\}$
encode the latitude and the longitude of the pierce point. $x_{3}$
is the zenith angle and $x_{4}$ is the azimuth angle. Define kernel
function
\[
\kappa(\boldsymbol{x},\boldsymbol{x}')\triangleq\exp\{-(\boldsymbol{x}-\boldsymbol{x}'){}^{\top}\Omega^{-1}(\boldsymbol{x}-\boldsymbol{x}')\}
\]
 where $\Omega=\mathrm{diag}(\omega_{1},\omega_{2},\omega_{3},\omega_{4})$
is a diagonal kernel parameter matrix. The kernel matrix induced by
$\kappa$ is $K\in\mathbb{R}^{n\times n}$ where $K_{i,j}\triangleq\kappa(\boldsymbol{x}^{(i)},\boldsymbol{x}^{(j)})$.
Denote $\boldsymbol{y}$ the double-difference value. Our GP model
aims to minimize the loss function
\begin{equation}
\min_{\boldsymbol{f}}\ \|A\boldsymbol{f}-\boldsymbol{y}\|^{2}+\lambda\boldsymbol{f}{}^{\top}K^{-1}\boldsymbol{f}\ .\label{eq:GP-model}
\end{equation}
 In every second, Eq. (\ref{eq:GP-model}) is solved using the data
points collected on the training stations. To predict on the testing
station, we predict the vTEC for each pierce point $f_{\mathrm{pred}}$
on the testing station by
\[
f_{\mathrm{pred}}=\sum_{i=1}^{n}\kappa(\boldsymbol{x}^{(i)},\boldsymbol{x}^{\mathrm{pred}})\boldsymbol{f}_{i}\ .
\]
 The double-difference value is followed immediately by $A_{\mathrm{pred}}f_{\mathrm{pred}}$
where $A_{\mathrm{pred}}$ is the quadruplet matrix of the testing
station.

In the offline stage, we tune model parameters on the 5 day training
data. The kernel parameter $\omega_{1}$ and $\omega_{2}$ are tuned
in range $[0.001,1000]$ . $\omega_{3}$ and $\omega_{4}$ are tuned
in range $[0.01,\pi/2]$ . $\lambda$ is tuned in range $[10^{-6},10^{6}]$
. To apply our algorithm, we first train Eq. (\ref{eq:GP-model})
on clean dataset. Then we apply the learned model on noisy dataset
to predict $\boldsymbol{y}^{\mathrm{pred}}$ and compute the absolute
residual $\boldsymbol{r}=|\boldsymbol{y}^{\mathrm{pred}}-\boldsymbol{y}|$.
For $c_{1}$ tuned in range $[0.5,8]$, if $\boldsymbol{r}_{i}\leq c_{1}\sum_{i=1}^{n}\boldsymbol{r}_{i}/n$,
the $i$-th noisy data point is reserved in the retraining step otherwise
is discarded. We tune the reweighting parameter $c_{2}$ in range
$[0.01,1]$ and $c_{3}$ in range $[0.001,0.1]$. The model Eq. (\ref{eq:GP-model})
is retrained using the filtered noisy dataset and the clean dataset
together.

For comparison, we considered the following four baseline methods
due to their popularity in robust regression literature: Hard Iterative
Thresholding (HIT) \cite{Bhatia2017}, Huber Regressor (Huber) \cite{Huber1992},
Iteratively Reweighted Least Squares (IRLS) \cite{Holland1977}, Random
Sample Consensus (RANSAC) \cite{Fischler1981}. Due to the computational
time constraint, for algorithms requiring iterative retraining, the
retraining is limited to 5 trials. For HIT, in each iteration we keep
$\rho$ percentage data points. That is, we always select the $\rho$
percentage data points with the smallest residual error for next iteration.
$\rho$ is tuned in range $[0.1,0.9]$. The Huber uses $\ell_{2}$
loss for small residual error and $\ell_{1}$ loss if the residual
error is larger than a threshold. We tune this threshold in range
$[1,5]$. For IRLS, we replace the $\ell_{2}$ loss in Eq. (\ref{eq:GP-model})
with $\ell_{1}$ norm. For RANSAC, in each iteration we keep $\rho$
percentage data points similarly as in HIT.

In offline tuning stage, for all methods we randomly select model
parameters uniformly in the tuning range and evaluate the RTK ratio.
The procedure is repeated 1000 times and the model with the highest
RTK ratio on training set is selected. The parameter tuning is carried
on a cluster with 20,000 CPU cores. Each method takes about 6 wall-clock
hours (13.9 CPU years) to find the best parameters. In Figure \ref{fig:rtk_ratio},
we report the RTK ratio of each method with the selected parameters
on 5 testing days. For each day we report the averaged RTK ratio over
$24\times3600=86,400$ seconds. The last figure reports the averaged
RTK ratio over all 5 days. The 95\% confidence intervals of all methods
are within $10^{-5}$ .

The numerical results show that directly training on the full dataset
could be harmful as the dataset contains lots of outliers --- even
a robust estimator is used. We see that all four baseline methods
will decrease the RTK ratio dramatically. This is not surprising as
most consistent robust methods require the noisy ratio small enough.
In our dataset, the noisy data could take up as many as 30\% of the
total data points, a number too large for most conventional robust
regression methods. The HIT achieves the best results among the four
baseline methods. This is interesting because the best provable breakdown
point of HIT is far below 1\%. Our experiment shows that HIT is able
to tolerate much higher breakdown point in real-world problems. The
fact that Huber does not work well might be due to the $\ell_{1}$
loss used in the Huber regression. Although $\ell_{1}$ loss is more
robust than $\ell_{2}$ loss, it is still affected by outliers. Therefore,
when the proportion of the outliers is as large as 30\%, $\ell_{1}$
loss cannot survive. The RTK ratio of RANSAC is barely above 50\%.
RANSAC requires lots of trials of random sampling in order to work.
However in our online system, we are only able to perform the random
sampling no more than 5 trials which makes RANSAC vulnerable. IRLS
performs poorly with RTK ratio less than half of FRR. This indicates
that IRLS is sensitive to outliers, as we showed in the discussion
of Theorem \ref{sec:Theoretical-Analysis}.

\section{Conclusion}

We propose to estimate the vTEC distribution in real-time by Gaussian
process regression with a three-step Filter-Reweight-Retrain algorithm
to enhance its robustness. We prove that the FRR is consistent and
has breakdown point asymptotically convergent to 1. We apply this
new method in our real-time high precision satellite-based positioning
system to improve the RTK ratio of our baseline Gaussian process model.
By large-scale numerical experiments, we demonstrate the superiority
of the proposed FRR against several state-of-the-art methods. The
FRR is more robust and accurate than conventional methods especially
in our system where conventional methods are fragile due to the high
corruption rate in GPS signal.


%
\section*{Acknowledgement}
  We acknowledge Qianxun Spatial Intelligence Inc. China for providing its high
  precision GPS devices and the grid computing system used in this work.

%
\bibliographystyle{plainnat}
\bibliography{refs}

%
\appendix

\section{Proof of Lemma \ref{lem:init-estimator-is-good-enough}\label{sec:Proof-of-Lemma-init-estimator-is-good-enough}}
\begin{proof}
Denote $n=|\mathcal{S}|$ and $X\in\mathbb{R}^{d\times n}$ where
the $i$-th column of $X$ is $\boldsymbol{x}^{(i)}\in\mathcal{S}$.
Similarly denote $\boldsymbol{y}\in\mathbb{R}^{d}$ to be the label
vector. The least square regression has solution
\begin{align*}
\boldsymbol{w}^{\mathrm{init}} & =(XX^{\top})^{-1}(X\boldsymbol{y})\\
 & =(\frac{1}{n}XX^{\top})^{-1}(\frac{1}{n}XX{}^{\top}\boldsymbol{w}^{*}-\frac{1}{n}X\boldsymbol{\xi})\ .
\end{align*}

To bound $\|(\frac{1}{n}XX^{\top})^{-1}\|_{2}$, since $\mathbb{E}\{\frac{1}{n}XX^{\top}\}=I$,
according to matrix Bernstein's inequality, with probability at least
$1-\eta$,
\begin{align*}
\|XX^{\top}-nI\|_{2}\leq & c\max\{c^{2}d\log^{2}(2d/\eta),c^{2}\sqrt{nd}\log(2d/\eta)\}\ .
\end{align*}
 Suppose
\begin{align*}
 & \sqrt{nd}\log(2d/\eta)\geq d\log^{2}(2d/\eta)\\
\Leftarrow & n\geq d\sqrt{\log(2d/\eta)}\ .
\end{align*}
 We get with probability at least $1-\eta$,
\begin{align*}
\|XX{}^{\top}-nI\|_{2}\leq & c^{3}\sqrt{nd}\log(2d/\eta)\ .
\end{align*}
 When
\[
n\geq4d\log^{2}(2d/\eta)/c^{6}\ ,
\]
\[
\|XX^{\top}-nI\|_{2}\leq n/2\ .
\]

To bound $\|X\boldsymbol{\xi}\|_{2}$, according to the assumption,
the noise distribution is symmetric,
\begin{align*}
\mathbb{E}\xi_{i}\boldsymbol{x}^{(i)} & =0\\
\max\|\xi_{i}\boldsymbol{x}^{(i)}\|_{2}\leq & \xi_{\max}c\sqrt{d\log(2d/\eta)}\ .
\end{align*}
\begin{align*}
 & \max\{\|\mathbb{E}[\xi_{i}^{2}\boldsymbol{x}^{(i)}\boldsymbol{x}^{(i)}{}^{\top}]\|_{2},\|\mathbb{E}[\xi_{i}^{2}\boldsymbol{x}^{(i)}{}^{\top}\boldsymbol{x}^{(i)}]\|_{2}\}\\
\leq & c^{2}\sigma^{2}d\log(2d/\eta)\ .
\end{align*}
 Therefore, with probability at least $1-\eta$,
\begin{align*}
\|X\boldsymbol{\xi}\|_{2}\leq & c\max\{c\xi_{\max}\sqrt{d}[\log(2d/\eta)]^{3/2},c\sigma\sqrt{nd}[\log(2d/\eta)]^{3/2}\}\ .
\end{align*}
 Suppose
\begin{align*}
 & c\sigma\sqrt{nd}[\log(2d/\eta)]^{3/2}\geq c\xi_{\max}\sqrt{d}[\log(2d/\eta)]^{3/2}\\
\Leftarrow & \sigma\sqrt{n}\geq\xi_{\max}\\
\Leftarrow & n\geq(\xi_{\max}/\sigma)^{2}\ ,
\end{align*}
 we get
\[
\|X\boldsymbol{\xi}\|_{2}\leq c^{2}\sigma\sqrt{nd}[\log(2d/\eta)]^{3/2}\ .
\]

In summary, with probability at least $1-\eta$,
\begin{align*}
\|\boldsymbol{w}^{\mathrm{init}}-\boldsymbol{w}^{*}\|_{2}\leq & \|(XX^{\top})^{-1}\|\|X\boldsymbol{\xi}\|_{2}\\
\leq & c^{2}\sigma\sqrt{d}[\log(2d/\eta)]^{3/2}/\sqrt{n}
\end{align*}
 provided
\begin{align*}
 & n\geq\max\{(\xi_{\max}/\sigma)^{2},4d\log^{2}(2d/\eta)/c^{6}\}\ .
\end{align*}
\end{proof}

\section{Proof of Lemma \ref{lem:trunc-noise-level} \label{sec:Proof-of-Lemma-trunc-noise}}
\begin{proof}
From filtering step in FRR (line 3 in Algorithm \ref{alg:FRR}), $\forall\boldsymbol{x}^{(i)}\in\hat{\mathcal{S}}'$,
\begin{align*}
\tau\geq r_{i} & =\left|y_{i}-\left\langle \boldsymbol{w}^{\mathrm{init}},\boldsymbol{x}^{(i)}\right\rangle \right|\\
 & =\left|\boldsymbol{w}^{*}{}^{\top}\boldsymbol{x}^{(i)}+\hat{\xi_{i}}-\boldsymbol{w}^{\mathrm{init}}{}^{\top}\boldsymbol{x}^{(i)}\right|\\
 & \geq|\hat{\xi_{i}}|-|(\boldsymbol{w}^{*}-\boldsymbol{w}^{\mathrm{init}}){}^{\top}\boldsymbol{x}^{(i)}|\ .
\end{align*}
 Since $\boldsymbol{x}^{(i)}$ is sub-gaussian random vector independent
to $(\boldsymbol{w}^{*}-\boldsymbol{w}^{\mathrm{init}})$, apply Bernstein's
concentration, we have with probability at least $1-2\eta$,
\[
|(\boldsymbol{w}^{*}-\boldsymbol{w}^{\mathrm{init}}){}^{\top}\boldsymbol{x}^{(i)}|\leq c\Delta\sqrt{\log(2/\eta)}\ .
\]
 Therefore we get
\[
|\hat{\xi_{i}}|\leq r_{i}+c\Delta\sqrt{\log(2/\eta)}\ .
\]
\end{proof}

\end{document}